\documentclass[a4paper]{svproc}

\usepackage{url}

\usepackage{mathptmx}          
\usepackage[T1]{fontenc}
\usepackage[utf8]{inputenc}
\usepackage{amsmath,amssymb}
\usepackage{graphicx}
\usepackage{booktabs}
\usepackage{multirow}
\usepackage{microtype}
\usepackage[table]{xcolor}
\usepackage{algorithm}
\usepackage{algpseudocode}
\usepackage{hyperref}
\usepackage{cleveref}
\usepackage[inline]{enumitem}
\usepackage{cleveref}
\usepackage{cite}
\usepackage[font=small]{caption}
\usepackage{xspace}

\newcommand{\Lam}{\Lambda}               
\newcommand{\err}{\hat{e}}              
\newcommand{\AoI}{\text{AoI}}  
\newcommand{\AoIbar}{\overline{\text{AoI}}}  
\newcommand{\Xnorm}{\mathcal{X}}         
\newcommand{\dmax}{d_{\max}}             
\newcommand{\vt}{v_t}                    
\newcommand{\vr}{v_r}                    
\newcommand{\rc}{r_c}                    
\newcommand{\rs}{r_s}                    
\newcommand{\area}{\mathcal{A}}          
\newcommand{\Eop}{\mathbb{E}}            
\newcommand{\lenc}{\lambda_{\text{enc}}} 
\newcommand{\li}{\lambda_{i}}            

\begin{document}
\mainmatter              
\title{A Kinetic Theory of Encounter-Based Information Propagation in Multi-Robot Systems}
\titlerunning{A Kinetic Theory of Encounter-Based Information Propagation in MRS}  
\author{Alkesh K. Srivastava \and
 Philip Dames }
\authorrunning{A.K. Srivastava and P. Dames}
%
\institute{Temple University, Philadelphia, PA, 19122 USA\\
\email{\{alkesh, pdames\}@temple.edu}}

\maketitle              
\begin{abstract}
Multi-robot systems cannot assume persistent network connectivity. We study this problem through target tracking, where performance depends on how quickly target information is sensed, transported through the team, and used before it becomes stale. When robots exchange information only through physical encounters, tracking becomes a kinetic information-transport problem: robot motion induces encounters, encounters carry target-state estimates, information age determines staleness, and stale information produces tracking error. This paper develops a kinetic theory of encounter-based information propagation and identifies three limits. The first is an \textit{access limit}---information cannot support team-level coordination unless it spreads beyond the robots that sensed it. The second is a \textit{staleness limit}---even propagated information loses value as the target moves. The third is a \textit{geometry limit}---when target motion outpaces information transport, tracking error approaches a saturation regime where communication improvements alone have diminishing returns. We evaluate the theory through large-scale simulations varying team size, operating area, communication range, and target speed. Results support the proposed access--staleness--geometry decomposition: communication coverage governs the access transition; once information is accessible, tracking error is shaped by target displacement; and this response is locally linear in restricted regimes but nonlinear over broader ranges because of sensing refreshes and bounded geometry. Across controlled sweeps and joint variation, the derived access and staleness coordinates reliably describe tracking performance. Together, these results establish a kinetic-theoretic framework for predicting and designing encounter-based multi-robot systems.
\end{abstract}

\section{Introduction}
\label{sec:intro}

Multi-robot systems increasingly operate in environments where communication is intermittent, local, or denied. In subterranean search-and-rescue, underwater robotics, cluttered urban settings, and contested environments, robots may exchange information only when they come into physical proximity. In such settings, mobility is not only a means of reaching sensing locations; it becomes the communication channel itself. 
Much of the literature on communication-constrained multi-robot systems focuses on how to plan, estimate, or coordinate under degraded communication, including communication-aware information gathering, distributed estimation, path-based sensing, and resilient multi-robot control~\cite{schwager2016multi,julian2012distributed,stachura2017communication,Otte.Sofge.TASE21,srivastava2022distributed,mcguire2024valuing,srivastava2025behaviorally,prorok2021beyond,liu2024multi, srivastava2026bayesian}. These works provide algorithms for operating when communication is limited. In this paper, we ask a more grounded question: \textit{when communication occurs only through proximity encounters, what determines whether information can propagate at all, whether it arrives too late to be useful, and when additional communication can no longer improve tracking?}

{Transport and diffusion ideas from physics have long been used to explain how local interactions produce macroscopic behavior, from mixing in physical systems to coverage and exploration in robotics. We draw on this perspective for encounter-based multi-robot coordination, where robot motion induces proximity encounters that move timestamped target estimates through the team. As these estimates age, target motion makes them stale and produces tracking error. In this view, encounter-based communication is a mobility-driven transport process shaped by motion, geometry, and mixing. This gives a principled route from motion and geometry to encounter rates, from encounter rates to Age of Information (AoI), and from AoI to task performance. Figure~\ref{fig:kinetic_analogy} summarizes this analogy and previews the three limiting mechanisms studied in this paper.}

\begin{figure}[t]
    \centering
    \includegraphics[width=0.95\textwidth]{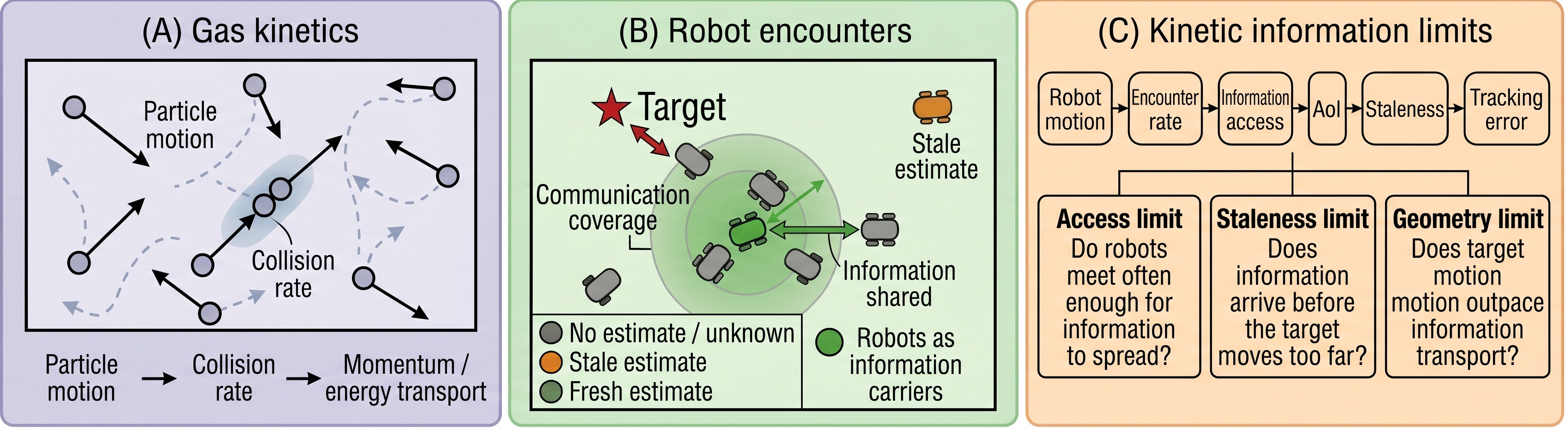}
    \caption{
    Kinetic analogy for encounter-based information propagation.
    (A) In classical kinetic theory, particle motion induces collisions that transport physical quantities such as momentum and energy.
    (B) In encounter-based robot teams, robot motion induces proximity encounters that transport timestamped target-state estimates.
    (C) This analogy yields a kinetic information-transport chain from robot motion to encounter rate, information access, Age of Information (AoI), staleness, and tracking error, producing access, staleness, and geometry limits.
    }
    \label{fig:kinetic_analogy}
\end{figure}

This perspective connects, but is distinct from, several existing lines of work. AoI measures information freshness in queues, wireless systems, sensor networks, and gossip processes~\cite{yates2021age,kosta2017age,hasan2025vaoi,barcis2020information}, while gossip and delay-tolerant networking study how information spreads through intermittent contacts~\cite{boyd2005gossip,haas2002gossip,spyropoulos2008efficient}. Percolation and random geometric graph models characterize connectivity transitions~\cite{stauffer2018introduction}. Kinetic and statistical-physics methods model collective systems through density, velocity, alignment, and collision-like interactions~\cite{aceves2019hydrodynamic,boltz2024kinetic,bellomo2026new}. On the other hand, we use kinetic encounter reasoning to connect physical robot motion to task-level tracking performance. The transported quantity is not mass or momentum, but timestamped target-state information.

The resulting kinetic theory predicts three limiting mechanisms that define a low-dimensional regime map for prediction and design. An \emph{access limit} arises when encounter opportunities are too sparse for target observations to spread through the team, motivating communication coverage as the coordinate for information accessibility. A \emph{staleness limit} arises when information propagates but arrives after the target has moved far enough that the estimate has lost value, motivating normalized staleness as the coordinate for information usefulness. A \emph{geometry limit} arises when target motion outpaces information transport, causing tracking error to approach a bounded saturation regime where communication improvements alone have diminishing returns. Together, these limits provide a kinetic interpretation of encounter-based coordination---\textit{information must first become accessible through robot encounters, then remain useful despite target motion, and finally remain bounded by task geometry.}

A related line of work seeks low-dimensional, interpretable predictors of multi-robot performance~\cite{xin2025towards}. Our paper follows that spirit, but with physically derived coordinates rather than purely data-driven ones. We use controlled simulations to isolate the proposed limits and a joint factorial parameter study with held-out parameter-level validation to test whether the kinetic regime variables predict performance under unseen team, environment, communication, and target-speed settings.

This paper makes three contributions. First, it formulates encounter-only communication as a kinetic information-transport process induced by robot motion. Second, it develops an access--staleness--geometry regime framework that explains when information can spread, when it remains useful, and when tracking error saturates against task geometry. Third, it validates this framework across controlled sweeps, mixed parameter variation, held-out deployment settings, and structured patrol motion, and translates the resulting limits into design guidance for encounter-based multi-robot systems.






\section{Problem Setup and Regime Variables}
\label{sec:problem}

We consider $N$ mobile robots operating in a square planar environment $\mathcal{E}\subset\mathbb{R}^2$ with side length $L$, area $\area=L^2$, and maximum separation $\dmax=L\sqrt{2}$. Robots move at characteristic speed $\vr$, while a single target moves at speed $\vt$. We do not assume a fixed ordering between $\vt$ and $\vr$; instead, $\vt$ is varied to study regimes in which target motion is slow, comparable to, or faster than information transport. Each robot directly observes the target only when it lies within sensing radius $\rs$; otherwise, it relies on the most recent timestamped target estimate obtained through direct sensing or encounter-based exchange. Communication occurs only through proximity encounters, i.e., two robots can exchange their stored estimates when their separation is at most the communication radius $\rc$, after which both retain the fresher estimate. 
\subsection{Performance and Freshness Metrics}

Let $\hat{\mathbf{p}}_i(t)$ denote robot $i$'s estimate of the target position, and let $\mathbf{p}_{\mathrm{t}}(t)$ denote the true target position. The primary task metric is the steady-state normalized tracking error
\begin{equation}
    \err =
    \frac{1}{\dmax}\,
    \Eop\!\left[
      \left\|\hat{\mathbf{p}}_i(t)-\mathbf{p}_{\mathrm{t}}(t)\right\|
    \right].
    \label{eq:norm_error}
\end{equation}
The information-freshness metric is Age of Information (AoI). Let $s_i(t)$ denote the timestamp of the estimate currently stored by robot $i$ at time $t$. The instantaneous AoI of robot $i$ is
\begin{equation}
    \AoI_i(t)=t-s_i(t).
    \label{eq:aoi_robot}
\end{equation}
We define the instantaneous team-averaged AoI as
\begin{equation}
    \AoI(t)=\frac{1}{N}\sum_{i=1}^{N}\AoI_i(t).
    \label{eq:aoi_team}
\end{equation}
We denote by $\AoIbar$ the steady-state time average of $\AoI(t)$, computed in simulation over the steady-state evaluation window.
\subsection{Regime Variables}

The goal of the kinetic theory is not to predict tracking error from every system parameter independently, but to identify coordinates that organize the dominant operating regimes. We therefore seek variables that answer two physical questions---\textit{can target information reach the team, and is that information still useful when it arrives?} These questions define the two primary regime coordinates used throughout the paper.

The first coordinate is communication coverage,
\begin{equation}
    \Lam = \frac{N\pi\rc^2}{\area},
    \label{eq:lambda_def}
\end{equation}
{which is the nominal total communication-disk area of the team, $N\pi\rc^2$, normalized by the environment area $\area$.} It increases with the number of robots and communication range, and decreases with operating area. We use $\Lam$ as the \emph{access coordinate} because it captures whether robots meet often enough for target observations to spread beyond the robots that directly sensed them.

The second coordinate is normalized staleness,
\begin{equation}
    \Xnorm = \frac{\vt\,\AoIbar}{\dmax},
    \label{eq:x_def}
\end{equation}
{which converts information age into a normalized target-displacement scale. Specifically, $\vt\AoIbar$ is the characteristic distance the target moves during the lifetime of an estimate, and $\dmax$ normalizes this distance by the size of the environment.} We use $\Xnorm$ as the \emph{staleness coordinate} because it measures how old information becomes relative to target motion and environment scale. {Thus, $\Xnorm$ is a physical proxy for the loss of tracking value caused by stale information.} Together, $\Lam$ and $\Xnorm$ form an access--staleness coordinate system for encounter-based tracking. The two mechanisms are coupled in the task but physically distinct. This separation is central to the kinetic interpretation developed in the following section.

\textbf{Problem Statement:} \textit{Given encounter-only communication, what physical transport mechanisms determine whether target information reaches the team, remains useful for tracking, and eventually saturates against geometry and target dynamics?} We study this question by evaluating whether the kinetic coordinates $(\Lam,\Xnorm)$ predict normalized tracking error across team, environment, communication, and target-speed settings.
\section{Kinetic Theory of Information Propagation}
\label{sec:theory}

This section develops a first-order kinetic theory of encounter-based information propagation, treating robots as moving carriers of timestamped target information and encounters as information-carrying collisions. The theory identifies three limiting mechanisms: (i)~\emph{access limit} from sparse encounters, (ii)~\emph{staleness limit} from finite information age, and (iii)~\emph{geometry limit} from bounded task space and target motion.

\subsection{Kinetic Transport Structure}
\label{sec:theory:dimensionless}

The regime variables in \cref{sec:problem} separate encounter-based tracking into two transport questions: whether target information becomes accessible through robot encounters, and whether that information remains fresh enough to be useful. The kinetic theory below explains the physical origin of these coordinates. Encounter geometry gives the access scaling for $\Lam$, AoI scaling connects encounter opportunity to information age, and normalized staleness $\Xnorm$ converts that age into target displacement.

At the regime level, we therefore write the tracking response generically as
\begin{equation}
    \err \approx f(\Lam,\Xnorm),
    \label{eq:regime_view}
\end{equation}
where $f$ is not assumed to be a universal closed-form law. Instead, it represents the empirical response organized by the two dominant mechanisms: access through encounters and loss of information value through staleness.

\subsection{Encounter Geometry and the Access Limit}
\label{sec:theory:encounter}

We quantify encounter opportunity by the swept-area construction shown in \Cref{fig:encounter_swept_area}. In the relative-motion frame, one robot is fixed and the other moves with relative velocity $\mathbf{v}_{\mathrm{rel}}$. Over a short interval $\Delta t$, the set of initial relative positions that lead to an encounter is approximated by a swept region of width $2\rc$ and length $\|\mathbf{v}_{\mathrm{rel}}\|\Delta t$.

\begin{figure}[t]
    \centering
    \includegraphics[width=0.9\linewidth]{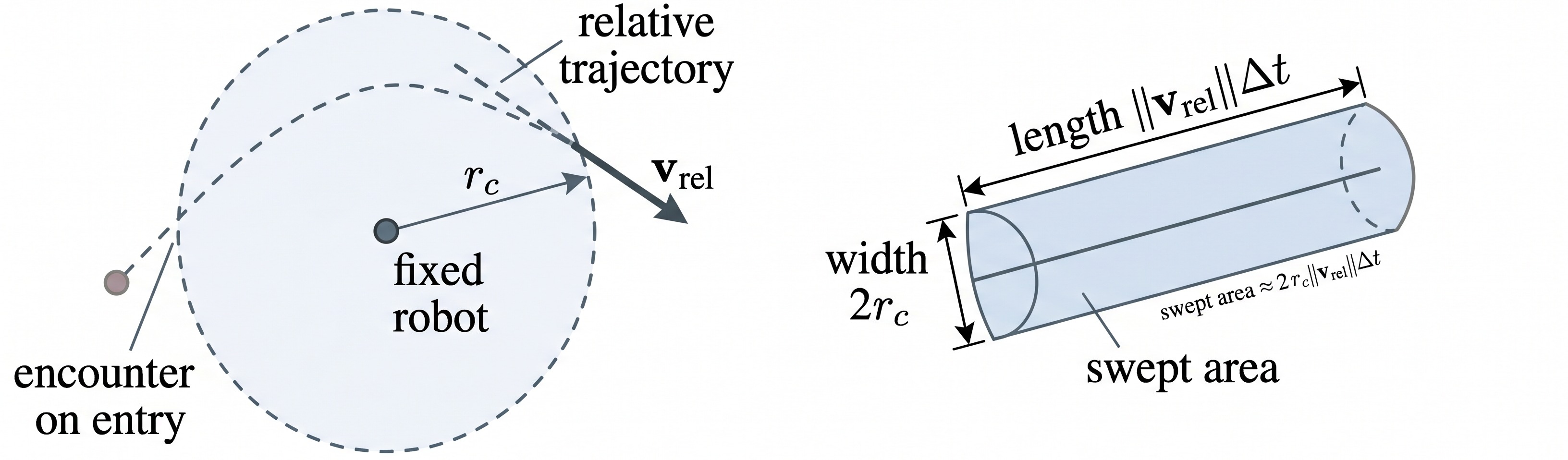}
    \caption{\textbf{Geometric derivation of the pairwise encounter rate.}
    In the relative-motion frame, an encounter occurs when the relative trajectory enters the communication disk of radius $\rc$. Over a short interval $\Delta t$, the swept region has width $2\rc$ and length $\|\mathbf{v}_{\mathrm{rel}}\|\Delta t$}.
    \label{fig:encounter_swept_area}
\end{figure}

\begin{lemma}[Pairwise encounter rate]
\label{lem:encounter_rate}
Consider two agents with independent isotropic headings and common speed $\vr$ in a large arena of area $\area$, with $\rc\ll\sqrt{\area}$. The pairwise encounter rate is
\begin{equation}
  \lenc =
  \frac{8\vr\rc}{\pi\area}
  + O\!\left(\frac{\rc}{\sqrt{\area}}\right).
  \label{eq:encounter_rate}
\end{equation}
\end{lemma}

\begin{proof}
In the relative frame, the speed between two equal-speed agents with heading difference $\phi$ is
\begin{equation}
  \|\mathbf{v}_1-\mathbf{v}_2\|
  =
  2\vr|\sin(\phi/2)|.
\end{equation}
For $\phi$ uniform on $[0,2\pi)$,
\begin{equation}
  \Eop[\|\mathbf{v}_{\mathrm{rel}}\|]
  =
  \frac{4\vr}{\pi}.
\end{equation}
As illustrated in \Cref{fig:encounter_swept_area}, during a short interval $\Delta t$, an encounter occurs if the initial relative position lies inside a swept region of width $2\rc$ and length $\|\mathbf{v}_{\mathrm{rel}}\|\Delta t$. Dividing the expected swept area by $\area$ and taking $\Delta t\rightarrow0$ yields \cref{eq:encounter_rate}. \qed
\end{proof}

For $N$ robots, the per-robot encounter opportunity scales as
\begin{equation}
  \li
  \approx
  (N-1)\lenc
  \approx
  \frac{8N\vr\rc}{\pi\area}
  =
  \frac{8\vr\Lam}{\pi^2\rc}.
  \label{eq:per_robot_rate}
\end{equation}

{This scaling gives the access interpretation of $\Lam$ as a normalized communication footprint that increases with team size and communication range, decreases with operating area, and controls the physical opportunity for information to move through the team. This access mechanism suggests a coverage transition. At low $\Lam$, target observations remain local because encounters are too sparse for team-wide propagation; at higher $\Lam$, information becomes broadly accessible and performance is limited mainly by staleness and target dynamics. We do not treat this transition as a fixed universal threshold, since mixed variation in $\vt$, $\rc$, $N$, and $L$ can broaden threshold-like behavior into an empirical transition band. The experiments estimate this band and test whether $\Lam$ organizes the access-limited regime.}
\subsection{AoI and the Staleness {Coordinate}}
\label{sec:theory:aoi}

Access alone does not guarantee useful information. Even when target estimates can propagate through encounters, they may arrive after the target has moved far enough that the estimate has lost value. {The role of AoI is to quantify this delay between when target information was generated and when it is used by a robot. We therefore use the encounter-rate scaling from \cref{sec:theory:encounter} to obtain a first-order estimate of information~age.}

From the per-robot encounter scaling in \cref{eq:per_robot_rate}, a robot receives encounter opportunities at rate $\li$. {In an idealized setting, suppose that each useful encounter provides an independent fresh target update and that useful updates arrive as a memoryless renewal process with rate proportional to $\li$. Then the characteristic inter-update time, and hence the mean AoI scaling, is}
\begin{equation}
  \Eop[\mathrm{AoI}] \sim \frac{1}{\li}
  =
  \frac{\pi^2\rc}{8\vr\Lam}.
  \label{eq:ideal_aoi_scaling}
\end{equation}

{This ideal expression gives the dependence on encounter opportunity but underestimates the age in realistic encounter propagation.} Encountered robots may already carry stale estimates, observations may require multiple encounters to reach distant robots, and robot motion may mix the team imperfectly. We collect these effects in a dimensionless inefficiency factor $\Gamma\geq 1$, yielding the first-order AoI scaling
\begin{equation}
  \Eop[\mathrm{AoI}]
  \approx
  \Gamma\,
  \frac{\pi^2\rc}{8\vr\Lam}.
  \label{eq:aoi_scaling_revised}
\end{equation}
Thus, \cref{eq:aoi_scaling_revised} should be interpreted as a \textit{first-order scaling law}---information age decreases with encounter opportunity and increases with stale relays, multi-hop propagation, and imperfect mixing.


\subsection{Staleness Response}
\label{sec:theory:staleness_response}

{We now describe the expected response of tracking error to $\Xnorm$ coordinate once information is accessible.} In the staleness-limited regime, older information should not improve tracking performance; therefore, normalized tracking error is expected to be nondecreasing with normalized staleness,
\begin{equation}
  \err = f(\Xnorm), \qquad f'(\Xnorm)\geq 0.
  \label{eq:monotone_staleness}
\end{equation}
{The condition $f'(\Xnorm)\geq 0$ means that error should increase or saturate as information becomes stale; it should not decrease solely because the estimate is older.}

For small or moderate stale displacement, a local expansion gives
\begin{equation}
  \err \approx e_0 + a\Xnorm,
  \label{eq:local_linear}
\end{equation}
which corresponds to the intuitive regime where additional target motion during information age produces approximately proportional tracking degradation.

This local relation is not expected to hold globally. Tracking error is bounded by the environment, direct sensing periodically refreshes some robots, and target motion can make displacement grow sublinearly over longer horizons. Thus, the theory predicts a qualitative response shape rather than a universal linear law: error should grow with staleness in the connected regime, but the response should eventually bend toward geometry-controlled saturation. The experiments test this prediction by comparing linear, sublinear, saturating, and piecewise response models.


\subsection{Geometry Limit}
\label{sec:theory:saturation}

The third limit appears when target motion outpaces information transport. In this regime, old estimates become nearly uninformative, and additional communication has diminishing returns unless it also reduces age or improves prediction. {The limiting reference comes from the largest-error case in which a robot's estimate is effectively decorrelated from the true target position. In that case, the estimate location and target location can be approximated as two independent uniform samples from the square domain.} For a square of side length $L$, the expected distance between two independent uniform points is approximately $0.5214L$. Normalizing by the maximum possible separation $\dmax=L\sqrt{2}$ gives
\begin{equation}
  \err_\infty
  =
  \frac{\Eop[\|\mathbf{x}_r-\mathbf{x}_t\|]}{\dmax}
  \approx
  \frac{0.5214L}{L\sqrt{2}}
  \approx 0.369.
  \label{eq:saturation_limit}
\end{equation}
{Thus, $\err_\infty$ is not a fitted parameter or an exact prediction for every experiment. It is a geometry-controlled reference value for the error expected when stale information carries little useful correlation with the current target position.} Observed averages need not equal this value exactly because direct sensing, finite trajectories, and correlated target motion can reduce error. The key implication is that tracking error cannot grow linearly with staleness indefinitely; once target motion dominates information transport, bounded geometry and intermittent direct sensing determine a saturation regime.
\section{Experimental Design}
\label{sec:experiments}

The experiments test the kinetic theory at two levels. First, controlled sweeps isolate the three predicted limits: access, staleness, and geometry. Second, a joint factorial parameter study tests whether the regime coordinates $(\Lam,\Xnorm)$ remain predictive when team size, environment scale, communication radius, and target speed vary together.

\paragraph{Simulation model.}
Robots and the target move in a square bounded domain with reflective walls. Unless otherwise specified, robots move with speed $\vr=2.0~\mathrm{m/s}$, while the target speed $\vt$ is varied across the controlled sweeps and factorial study. Each simulation runs for $200~\mathrm{s}$ with timestep $0.1~\mathrm{s}$, and steady-state metrics are computed over the final portion of the run. Robots directly observe the target within sensing radius $\rs$ and otherwise rely on the freshest timestamped estimate obtained through sensing or proximity encounters. In the encounter-only protocol, robots exchange estimates only when their separation is below $\rc$ and retain the fresher estimate. For each condition, we record steady-state normalized tracking error $\err$, mean AoI, informed fraction, fresh-information fraction, and mean neighbor count, and compute the kinetic regime coordinates $\Lam$ and $\Xnorm$ using \cref{eq:lambda_def,eq:x_def}.
\paragraph{Mechanism-isolation sweeps.}
The controlled campaign varies one parameter at a time around the base configuration $N=40$, $L=100~\mathrm{m}$, $\rc=10~\mathrm{m}$, $\vt=2.0~\mathrm{m/s}$, $\vr=2.0~\mathrm{m/s}$, and $\rs=5.0~\mathrm{m}$. Sweeps over $N$, $L$, and $\rc$ test the access limit by changing encounter opportunity, while sweeps over $\vt$ test the staleness and geometry limits by changing target displacement during information age. The sweep ranges overlap the factorial study, so the controlled experiments isolate individual mechanisms while the factorial study tests their joint effect. A focused coverage sweep resolves the access transition, and no-sharing and oracle baselines bound the performance of the encounter-only protocol.
\paragraph{Joint factorial study.}
To test whether the access--staleness structure survives mixed variation, we run a factorial study over $N\in\{10,20,40,80\}$, $L\in\{50,100,150,200\}~\mathrm{m}$, $\rc\in\{5,10,15,20\}~\mathrm{m}$, and $\vt\in\{0.5,1,2,4\}~\mathrm{m/s}$, with $\vr=2.0~\mathrm{m/s}$ and $\rs=5.0~\mathrm{m}$ fixed. The resulting $256$ mixed conditions, with $10$ seeds each, test whether $\Lam$ and $\Xnorm$ organize performance under joint parameter variation.
\paragraph{Validation tests.}
We test whether $\Xnorm$, $\Lam$, or their joint coordinates $(\Lam,\Xnorm)$ best predict normalized tracking error. For each feature set, we fit the same quadratic ridge-regression model to factorial-study condition means and evaluate five-fold cross-validation using RMSE, MAE, and $R^2$. We also perform held-out parameter validation by withholding one level of $N$, $L$, $\rc$, or $\vt$ and fitting on the remaining levels, which tests generalization to unseen team sizes, environment scales, communication ranges, and target speeds.

\section{Results}
\label{sec:results}
\begin{figure}[t]
  \centering
  \includegraphics[width=\textwidth]{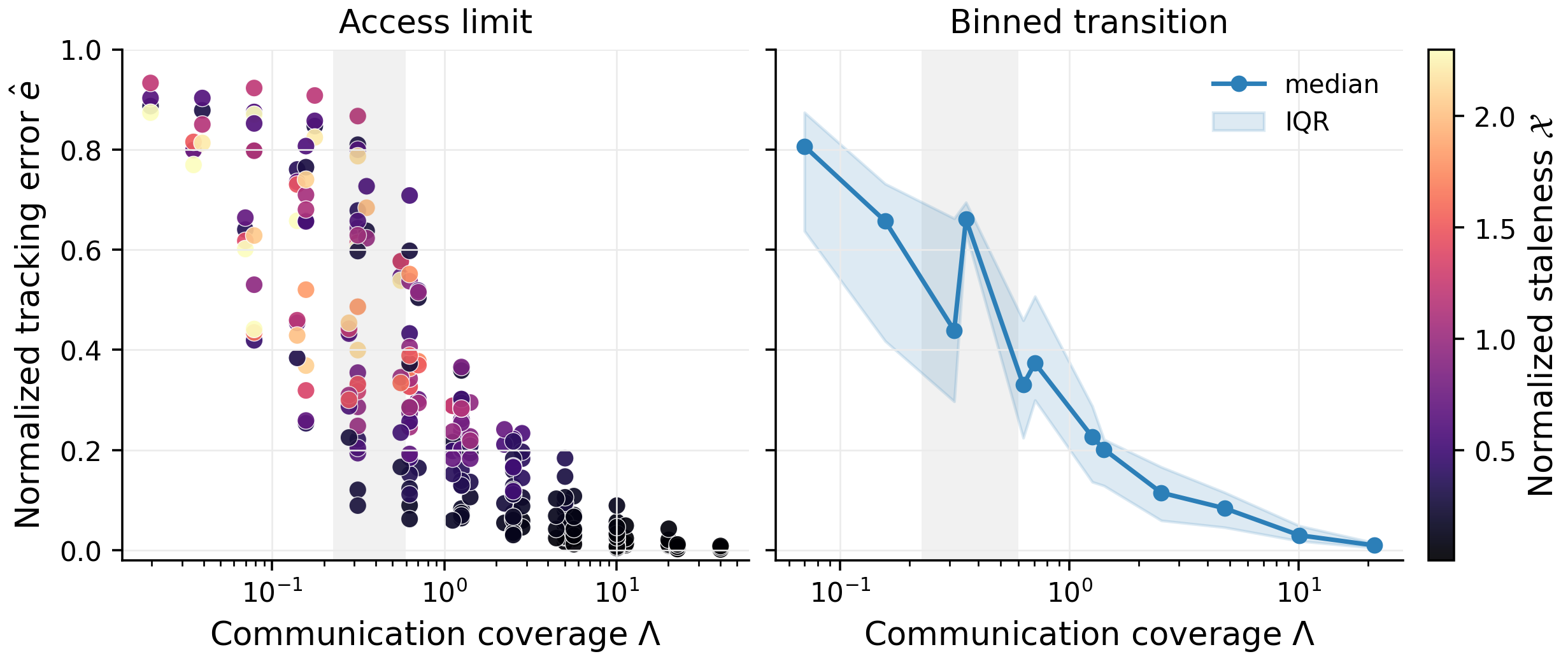}
  \caption{\textbf{Access limit under mixed variation.}
    Normalized Tracking error decreases with communication coverage $\Lam$ under joint parameter variation. Points are colored by normalized staleness $\Xnorm$, and the shaded region marks the empirical transition band $\Lam\in[0.226,0.597]$.}  \label{fig:access_transition_band}
\end{figure}

The results test the kinetic theory of encounter-based information propagation through the three predicted limits: access, staleness, and geometry. We first evaluate whether communication coverage $\Lam$ organizes the access transition, then test whether the joint coordinates $(\Lam,\Xnorm)$ predict tracking error better than either coordinate alone. We then examine the staleness response and the geometry-controlled saturation regime.

\subsection{Access Limit: Coverage Controls Information Availability}
\label{sec:results:transition}

\Cref{fig:access_transition_band} evaluates the access limit under joint parameter variation. At low communication coverage, tracking error remains high because target observations tend to stay local. In other words, robots do not encounter one another often enough for information to propagate through the team. As $\Lam$ increases, tracking error decreases sharply and then levels off, indicating that information has become broadly accessible and that other mechanisms begin to dominate.

The transition is not best interpreted as a fixed universal threshold. Estimating the coverage-dependent transition without assuming a sharp boundary gives an empirical band $\Lam\in[0.226,0.597]$ with midpoint $0.367$. This supports the access-limit interpretation: coverage governs whether information can spread through the team, but the apparent boundary broadens when target speed, environment scale, and staleness vary together.

\subsection{Access--Staleness Regime Map}
\label{sec:results:regime_map}

\Cref{fig:access_staleness_regime_map} shows the empirical regime map induced by the kinetic coordinates. Each point is a factorial-study condition averaged over seeds. Low $\Lam$ and high $\Xnorm$ correspond to high tracking error, while high $\Lam$ and low $\Xnorm$ correspond to low tracking error. This pattern matches the two-stage structure predicted by the theory---\textit{target information must first become accessible through encounters, and then it must remain useful despite target motion during the information age}. The map also clarifies why coverage is expected to dominate many conditions. If $\Lam$ is too low, the system fails at the access stage and staleness cannot fully explain performance. Once access is sufficient, however, variation in $\Xnorm$ captures how much value is lost while information propagates.

\begin{figure}[t]
  \centering
  \includegraphics[width=0.65\textwidth]{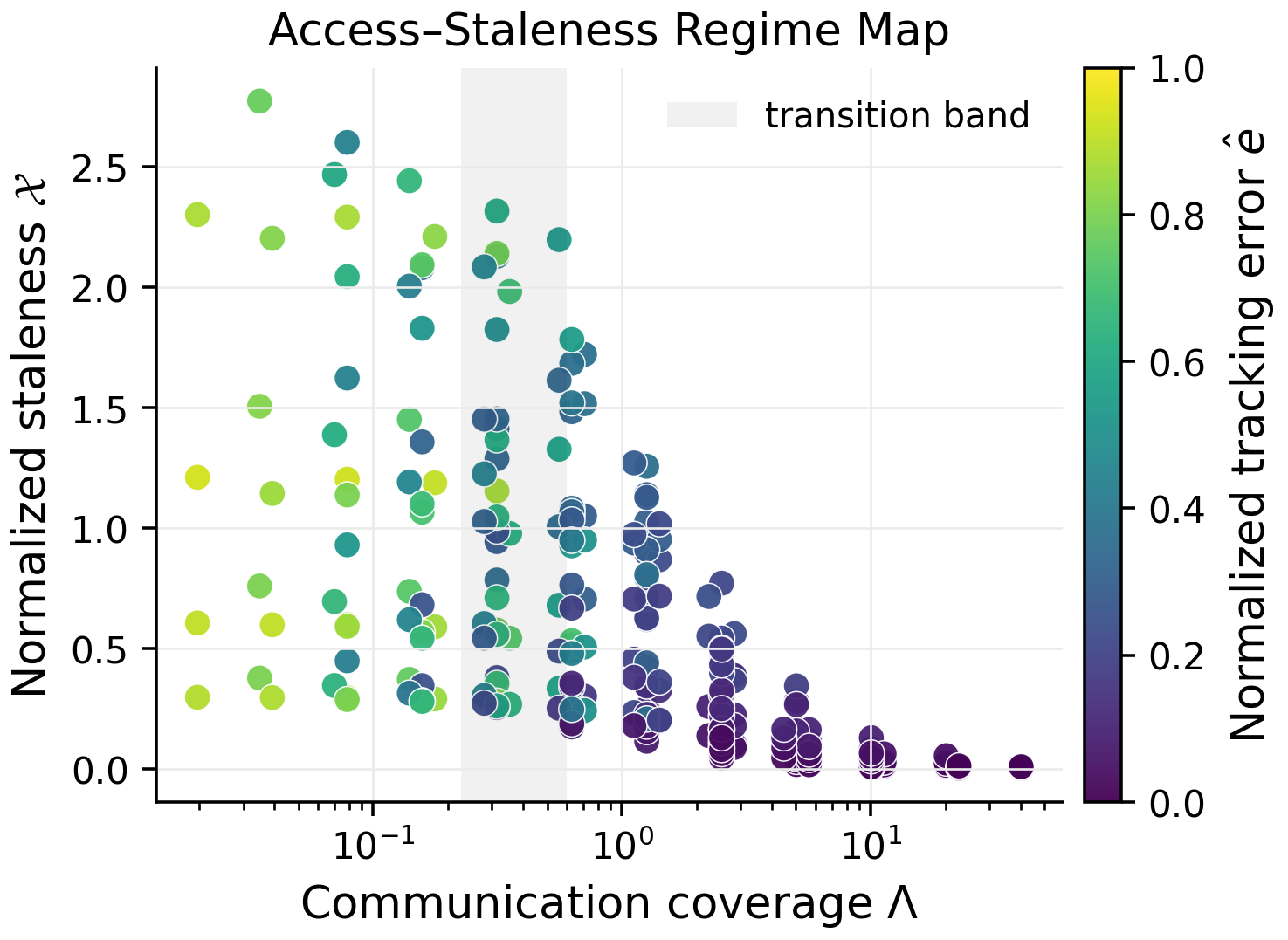}
  \caption{\textbf{Access--staleness regime map.}
  Points show factorial-study conditions in $(\Lam,\Xnorm)$, colored by normalized tracking error $\err$. High error occurs with poor coverage or stale information, while low error requires both higher coverage and lower staleness. The shaded band marks the empirical access-transition region.}
  \label{fig:access_staleness_regime_map}
\end{figure}

\subsection{Predictive Validation of the Regime Coordinates}
\label{sec:results:prediction}

We next test whether the kinetic coordinates predict normalized tracking error across mixed parameter settings. \Cref{tab:predictive_comparison} and \Cref{fig:predictive_comparison} compare three feature sets: $\Xnorm$, $\Lam$, and $(\Lam,\Xnorm)$. The joint coordinates achieve the best condition-level cross-validation performance, with RMSE $0.133$, MAE $0.096$, and $R^2=0.760$. Coverage alone is close, with RMSE $0.139$ and $R^2=0.739$, while staleness alone performs substantially worse. \textbf{This result is physically meaningful rather than disappointing.} Access is a prerequisite limit. When information cannot spread through the team, coverage alone explains much of the error. The additional improvement from $(\Lam,\Xnorm)$ shows that staleness still contributes once information becomes accessible.

\begin{table}[t]
  \centering
    \caption{\textbf{Condition-level validation.}
    Five-fold cross-validation on factorial-study condition means.}
  \label{tab:predictive_comparison}
  \begin{tabular}{lccc}
    \toprule
    Feature set & CV RMSE & CV MAE & CV $R^2$ \\
    \midrule
    $\Xnorm$ & 0.215 & 0.167 & 0.376 \\
    $\Lam$ & 0.139 & 0.106 & 0.739 \\
    $(\Lam,\Xnorm)$ & \textbf{0.133} & \textbf{0.096} & \textbf{0.760} \\
    \bottomrule
  \end{tabular}
\end{table}

\begin{figure}[t]
  \centering
  \includegraphics[width=\textwidth]{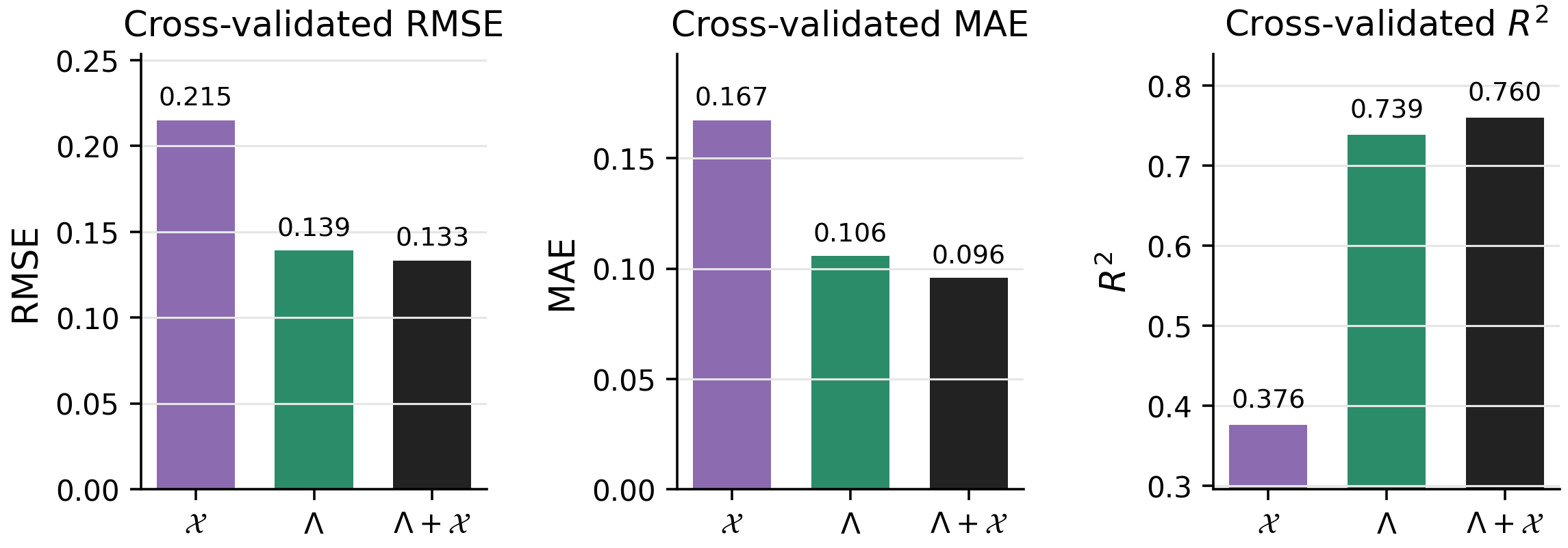}
  \caption{\textbf{Predictive comparison of kinetic coordinates.}
  Cross-validated RMSE, MAE, and $R^2$ for factorial-study condition means using the same quadratic regression form for each feature set. The joint coordinates $(\Lam,\Xnorm)$ provide the best prediction, while $\Lam$ alone remains close because access is the dominant failure mode across much of the tested parameter grid.}
  \label{fig:predictive_comparison}
\end{figure}

As a stricter test, we use held-out parameter validation. For each parameter $N$, $L$, $\rc$, and $\vt$, the regression is fit using three parameter levels and evaluated on the held-out fourth level. \Cref{tab:heldout_validation} shows the same overall ordering: $(\Lam,\Xnorm)$ performs best, with RMSE $0.155$ compared with $0.161$ for $\Lam$ alone. The improvement is modest but consistent with the theory---\textit{coverage determines whether information can propagate, while staleness adds explanatory power once access is available.}

\begin{table}[t]
  \centering
  \caption{\textbf{Held-out parameter-level validation.}
    Metrics are averaged over held-out levels of $N$, $L$, $\rc$, and $\vt$.}
  \label{tab:heldout_validation}
  \begin{tabular}{lccc}
    \toprule
    Feature set & Test RMSE & Test MAE & Test $R^2$ \\
    \midrule
    $\Xnorm$ & 0.235 & 0.188 & -0.215 \\
    $\Lam$ & 0.161 & 0.128 & 0.353 \\
    $(\Lam,\Xnorm)$ & \textbf{0.155} & \textbf{0.118} & \textbf{0.438} \\
    \bottomrule
  \end{tabular}
\end{table}
\subsection{Robustness to Structured Search Motion}
\label{sec:results:motion_robustness}

{The main experiments use correlated random-walk motion, which gives a clean kinetic setting for deriving the encounter-rate model. To test whether the regime picture is an artifact of this particle-like motion, we repeated the analysis with independent randomized lawnmower patrols, in which each robot follows a boustrophedon coverage path with randomized initial position, sweep direction, lane offset, and phase. This policy is structured, decentralized, and communication-free in which robots follow coverage-style sweeps rather than isotropic random trajectories, while the sensing model, target motion, encounter exchange rule, AoI computation, and tracking-error metric remain unchanged.}

{\Cref{fig:motion_policy_robustness} shows that the access--staleness structure persists under this structured patrol. In both motion policies, normalized tracking error decreases as communication coverage $\Lam$ increases. For randomized lawnmower patrol, $\log_{10}\Lam$ remains strongly negatively correlated with tracking error, with Spearman correlation $\rho=-0.734$ and a $95\%$ bootstrap confidence interval $[-0.842,-0.576]$; median error decreases from $0.262$ in the lowest-$\Lam$ third to $0.059$ in the highest-$\Lam$ third. The regime maps show the same qualitative organization in which high-error conditions occur at low $\Lam$ or high $\Xnorm$, while low-error conditions concentrate at higher $\Lam$ and smaller $\Xnorm$. Thus, changing the motion policy changes the effective mixing and AoI levels, but does not remove the kinetic access--staleness structure for these two distinct policies.}
\begin{figure*}[t]
  \centering
  \includegraphics[width=\textwidth]{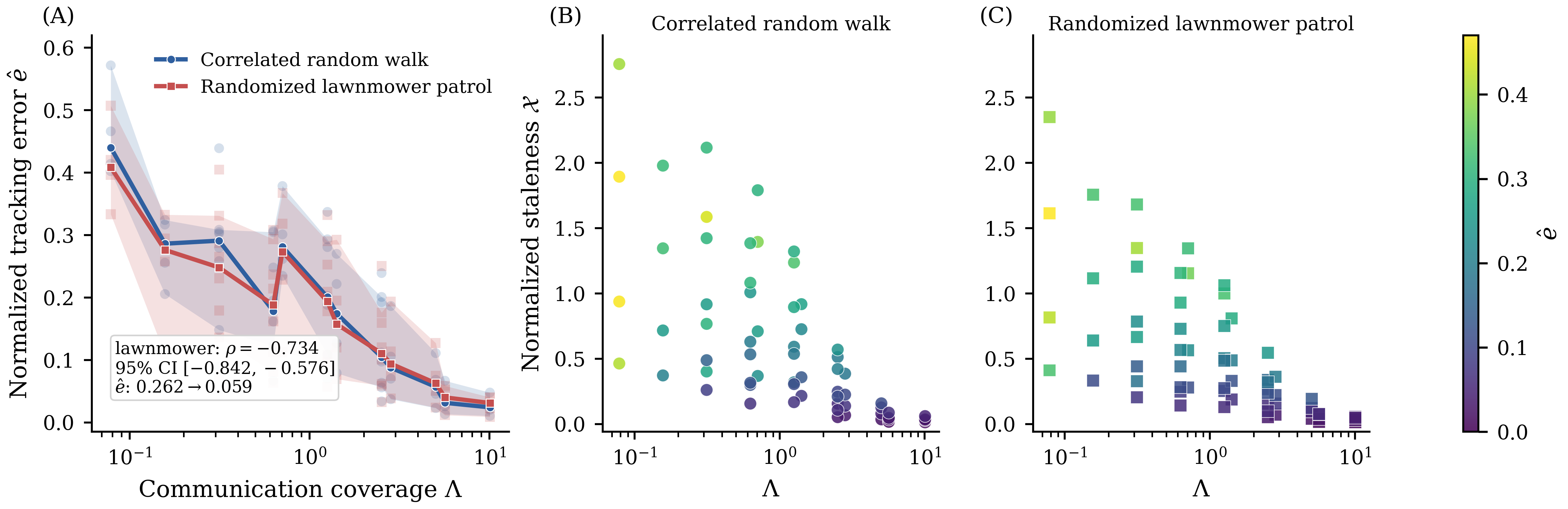}
    \caption{\textbf{Motion-policy robustness.}
    (A) Tracking error decreases with communication coverage $\Lam$ under both correlated random walk and randomized lawnmower patrol. (B,C) The same access--staleness structure appears under both policies: high error occurs at low $\Lam$ or high $\Xnorm$, while low error concentrates at higher $\Lam$ and lower $\Xnorm$.}
  \label{fig:motion_policy_robustness}
\end{figure*}
\subsection{Staleness Limit: Local Linearity and Global Saturation}
\label{sec:results:staleness}

The theory predicts local linearity: once information is accessible, additional target motion during information age should produce approximately proportional tracking degradation,
\begin{equation}
    \err \approx e_0 + a\Xnorm .
    \label{eq:results_linear_staleness}
\end{equation}
To test whether this local approximation holds globally, we compare it against representative sublinear and saturating response families,
\begin{align}
    \err &= e_0 + a\sqrt{\Xnorm}, \label{eq:results_sqrt_staleness}\\
    \err &= e_0 + a\Xnorm^b, \qquad 0<b<1, \label{eq:results_power_staleness}\\
    \err &= e_0 + (e_\infty-e_0)\left(1-\exp(-k\Xnorm)\right), \label{eq:results_exp_staleness}
\end{align}
together with logistic saturation and piecewise-linear fits. The logistic model captures smooth saturation, while the piecewise-linear model allows the slope to change across staleness regimes. These families are not separate theories; they are empirical probes of the same prediction that error should grow with staleness locally but bend toward saturation over broader operating regimes.

\begin{table}[t]
  \centering
  \caption{\textbf{Staleness response model families.}
  Group cross-validated RMSE on factorial-study data.}
  \label{tab:staleness_family}
  \begin{tabular}{lcc}
    \toprule
    Model & All conditions & $\Lam>0.60$ \\
    \midrule
    Linear & 0.257 & 0.139 \\
    Square-root & \textbf{0.243} & 0.132 \\
    Power law & 0.244 & 0.132 \\
    Exponential saturation & 0.245 & 0.131 \\
    Logistic & 0.253 & 0.132 \\
    Piecewise linear & 0.246 & \textbf{0.130} \\
    \bottomrule
  \end{tabular}
\end{table} 
\begin{figure}[t]
  \centering
  \includegraphics[width=\textwidth]{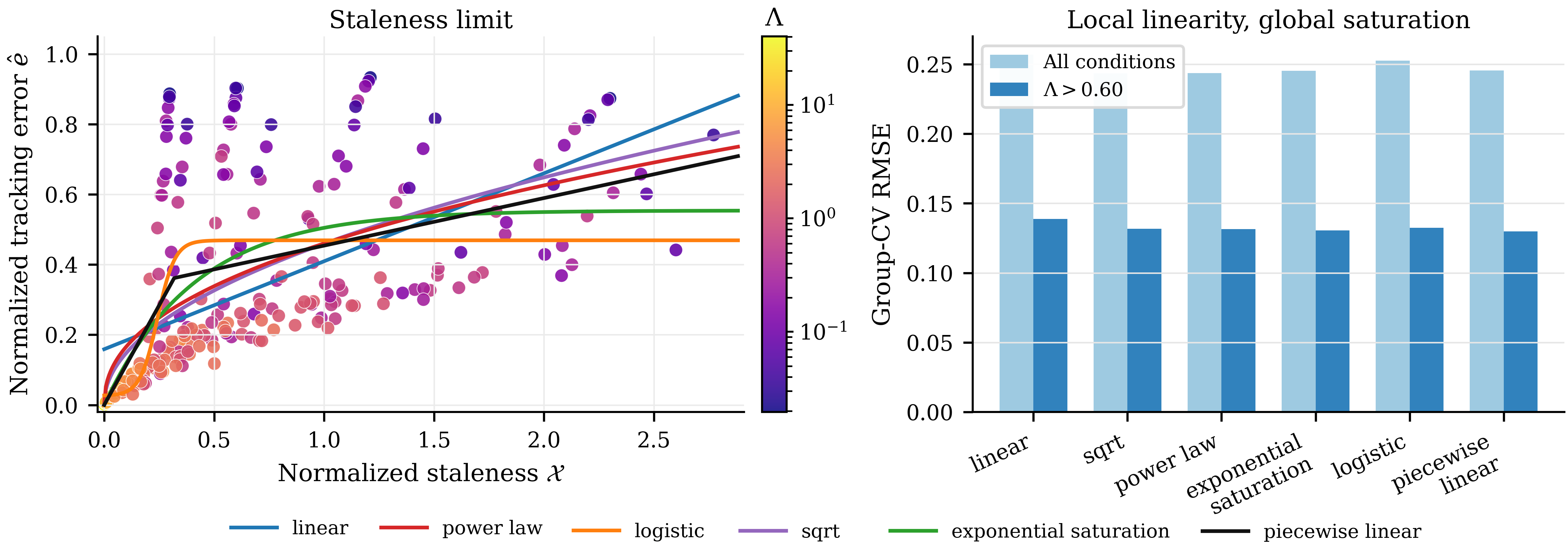}
  \caption{\textbf{Staleness response.}
  Normalized tracking error versus normalized staleness with candidate response families and cross-validated RMSE. Nonlinear saturating models provide a better fit than a global linear response, indicating geometry-driven saturation of the staleness response.}
  \label{fig:staleness_response_family}
\end{figure}

\Cref{fig:staleness_response_family} and \Cref{tab:staleness_family} show that a global linear model is not the best description of the staleness response. Over all conditions, the square-root family gives the lowest cross-validated RMSE, $0.243$ compared with $0.257$ for the linear model. A paired condition-bootstrap comparison gives an RMSE reduction of $0.0136$ with $95\%$ confidence interval $[0.0101,0.0168]$. In the connected regime $\Lam>0.60$, the piecewise-linear model gives the lowest RMSE, $0.130$ compared with $0.139$ for the linear model, with RMSE reduction $0.0089$ and $95\%$ confidence interval $[0.0056,0.0120]$. Because the leading nonlinear, saturating, and piecewise families are close to one another, we interpret the result as evidence that the staleness response bends away from global linearity, not as identification of a unique closed-form law.

\subsection{Geometry Limit: Saturation Under Fast Target Motion}
\label{sec:results:saturation}

\Cref{fig:geometry_saturation} evaluates the geometry limit using the target-speed sweep. As target speed increases, normalized tracking error rises but eventually approaches a bounded regime. The dotted line shows the geometry-only reference $\err_\infty\approx0.369$, corresponding to the expected normalized distance between two independent uniform points in a square domain. The local-encounter curve remains below this reference because direct sensing, finite trajectories, and target-motion correlations retain some information. The key implication is that communication improvements alone have diminishing returns once target motion outpaces information transport; further gains require reducing age, improving prediction, or strengthening sensing.

\begin{figure}[t]
  \centering
  \includegraphics[width=\textwidth]{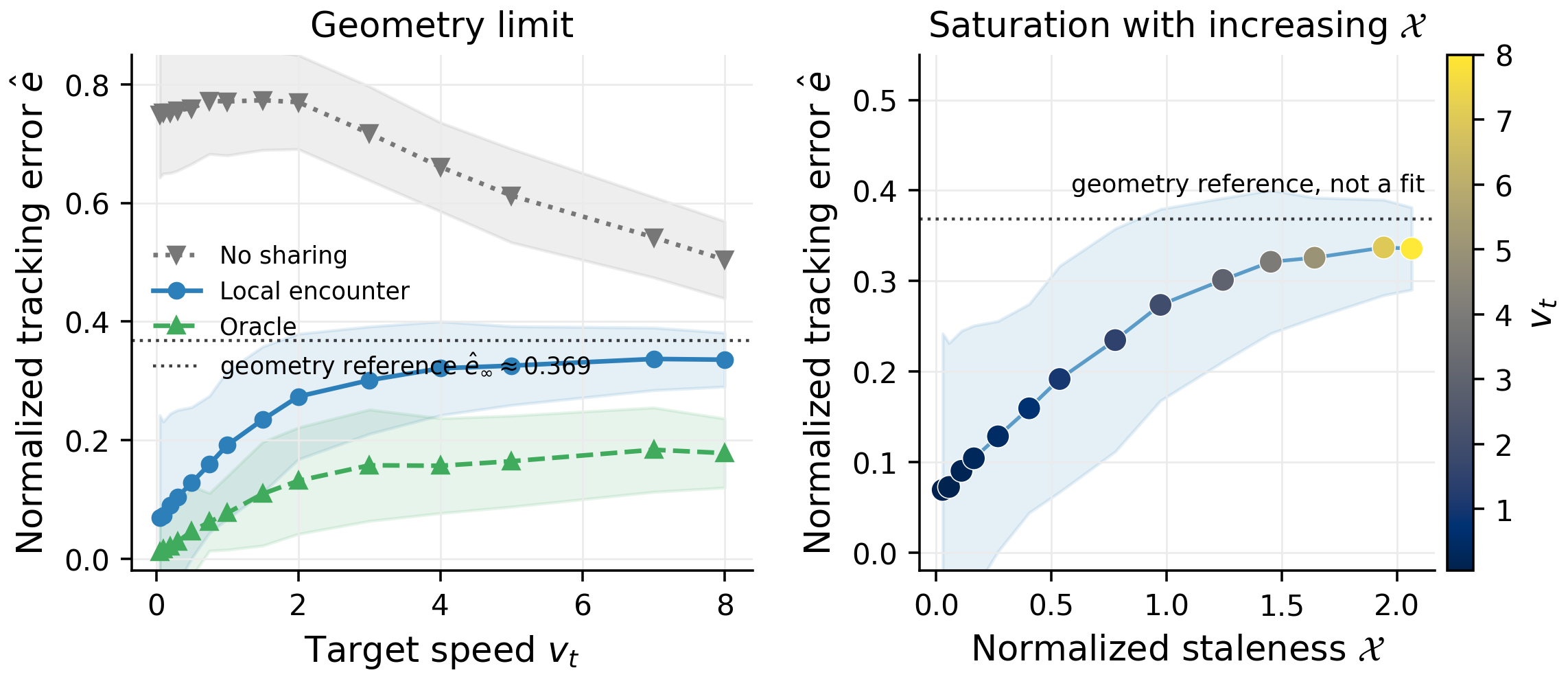}
  \caption{\textbf{Geometry limit.}
  Tracking error rises with target speed and normalized staleness but approaches a bounded regime. The dotted line marks the geometry reference $\err_\infty\approx0.369$.}
  \label{fig:geometry_saturation}
\end{figure}

\subsection{Design Choices from the Kinetic Limits}
\label{sec:results:design}

The access--staleness--geometry decomposition is not only explanatory; it also gives an engineering procedure for diagnosing and improving encounter-based tracking systems. A designer can first estimate the communication coverage $\Lam$ from team size, communication radius, and operating area, and estimate normalized staleness $\Xnorm$ from measured AoI and target speed. The location of the system in the $(\Lam,\Xnorm)$ regime map then indicates which intervention is most likely to improve performance.

If the system is \emph{access-limited}, $\Lam$ is low and target information does not spread beyond the robots that directly sense it. In this regime, improving the estimator or predictor alone is unlikely to solve the dominant failure mode because most robots never receive useful information. The design should instead increase encounter opportunity by adding robots, increasing communication radius, reducing operating area, improving motion mixing, or introducing relay and rendezvous behaviors.

If the system is \emph{staleness-limited}, $\Lam$ is sufficient for information to propagate but $\Xnorm$ remains high. In this regime, the problem is not whether information reaches the team, but whether it arrives before the target has moved too far. The design should reduce AoI or compensate for it by prioritizing fresh updates, increasing sensing frequency, shortening update paths, improving relay policies, or predicting target motion during information propagation.

{If the system is \emph{geometry-limited}, tracking error approaches the geometry-controlled reference $\err_\infty\approx0.369$ because target motion has largely decorrelated stale estimates from the current target position. In this regime, additional communication has diminishing returns unless it also reduces information age or improves sensing effectiveness. The design should therefore shift from merely increasing connectivity to improving target-dynamics prediction, increasing robot speed, enlarging sensing radius, or using active tracking policies that keep robots close to the target.}

Thus, the kinetic theory suggests a sequential design logic---first ensure information can become accessible, then ensure it remains fresh enough to be useful, and finally recognize when bounded geometry and target dynamics dominate the achievable error. This turns the regime map into a practical diagnostic tool for communication-constrained multi-robot tracking.

\section{Conclusion}
\label{sec:conclusion}

This paper introduces a kinetic theory of encounter-based information propagation for multi-robot systems operating in a communication-limited environment. When robots exchange information only through physical proximity, robot motion induces information-carrying encounters, analogous to collisions in kinetic theory. This viewpoint connects motion and geometry to encounter rates, encounter rates to information age, and information age to tracking error.

The theory identifies three limits: an \emph{access limit}, which determines whether target information can spread beyond the robots that directly sensed it; a \emph{staleness limit}, which determines whether propagated information remains useful after the target has moved; and a \emph{geometry limit}, which determines when bounded task space and target motion cause error to saturate. The results presented in this paper support this decomposition in which communication coverage $\Lam$ organizes the access transition, normalized staleness $\Xnorm$ captures loss of information value once access is available, and the joint coordinates $(\Lam,\Xnorm)$ provide the best predictive description of tracking error. The resulting design logic is sequential---first ensure information can spread, then ensure it remains fresh, and finally recognize when geometry and target dynamics dominate achievable performance.

The present validation has the limitation that it is simulation-based and considers a single target, homogeneous robots, simple sensing, correlated random-walk motion, and a freshest-estimate exchange protocol. Future work should derive the propagation inefficiency factor $\Gamma$ from motion mixing and gossip-path statistics, extend the framework to multi-target and heterogeneous teams, and test the access--staleness--geometry map on physical robot teams with realistic communication constraints. Overall, the kinetic theory provides a physics-grounded framework for predicting, interpreting, and designing encounter-based multi-robot coordination.
\section*{Acknowledgment}This work was funded by NSF grant CNS-2143312.

\bibliographystyle{ieeetr}
\bibliography{paper/references}

\end{document}